%% file: templateArxiv.tex
\documentclass{article}

\usepackage{PRIMEarxiv}
\input{macro.tex}

\usepackage{natbib}
\usepackage{amsmath}
\usepackage{amsthm}
\usepackage{amssymb}
\usepackage{algorithmic}
\usepackage{algorithm}

\usepackage{amsmath}
\usepackage{amssymb}
\usepackage{mathtools}
\usepackage{amsthm}
\usepackage{accents}
\usepackage{bbm}
\usepackage{multicol,lipsum,xparse}

\DeclareMathOperator{\spn}{span}

\usepackage[utf8]{inputenc} 
\usepackage[T1]{fontenc}    
\usepackage[hidelinks]{hyperref}
\usepackage{url}            
\usepackage{booktabs}       
\usepackage{amsfonts}       
\usepackage{nicefrac}       
\usepackage{microtype}      
\usepackage{lipsum}
\usepackage{fancyhdr}       
\usepackage{graphicx}       
\graphicspath{{media/}}     

\theoremstyle{plain}
\newtheorem{theorem}{Theorem}[section]
\newtheorem{proposition}[theorem]{Proposition}
\newtheorem{lemma}[theorem]{Lemma}
\newtheorem{corollary}[theorem]{Corollary}
\theoremstyle{definition}
\newtheorem{definition}[theorem]{Definition}

\theoremstyle{remark}

\title{Optimal Thresholding Linear Bandit}

\author{
  Eduardo Ochoa Rivera, Ambuj Tewari \\
  University of Michigan, \\
  Ann Arbor, USA \\
  \today
}

\begin{document}
\maketitle

\begin{abstract}
We study a novel pure exploration problem: the $\epsilon$-Thresholding Bandit Problem (TBP) with fixed confidence in stochastic
linear bandits. We prove a lower bound for the sample complexity and extend an algorithm designed for Best Arm Identification in the linear case to TBP that is asymptotically optimal.
\end{abstract}


\section{Introduction}
\label{introduction}

The Thresholding Bandit Problem (TBP) \cite{locatelli2016thresbandit, kano2018thresbandit} represents a specific combinatorial pure exploration problem \cite{chen2013combinatorialbandit} within the stochastic bandit setting. In this problem, the learner's objective is to identify all arms whose mean reward is above a specified threshold $\rho \in \bbR$.

The study by \citet{kano2018thresbandit} emphasizes that in certain contexts, such as personalized recommendations, pursuing the absolute best option might not always align with the most advantageous objective. Rather, the pursuit of viable alternatives could yield greater benefits. This concept extends to domains like scientific discoveries, such as in chemical reactions, where the primary focus might not solely revolve around identifying the single best reaction, but rather maximizing the discovery of multiple favorable reactions. Given the perceived value in practical applications, our objective is to advance our understanding of TBP, particularly within the framework of structured bandits.

On the other hand, the stochastic linear bandit \cite{auer2002linearbandit} is a sequential decision-making problem that extends the classical stochastic Multi-Armed Bandit \cite{robbins1952bandit}. In this scenario, the arms' mean rewards follow a linear model with unknown parameters. While much of the analysis of this problem has focused on the regret minimization objective \cite{rusmevichientong2010linearbandit, abbasiyadkori2011linearbandit}, some research has addressed it within the context of the best-arm identification (BAI) problem \cite{soare2014bailinear, soare2015bailinear, jedra2020optbailinear, degenne2020pureexp}.

\subsection{Contributions}
In this paper we explore the understudied $\epsilon$-TBP in stochastic linear bandits and make the following key contributions.
\begin{itemize}
    \item We prove an instance-specific lower bound for the expected sample complexity of any correct algorithm. 

    \item We extend the algorithm of \citet{jedra2020optbailinear} to the TBP and prove it is asymptotically optimal. The main difference is that their algorithm focuses on sampling arms to reduce uncertainty, particularly for arms that are harder to distinguish from the optimal arm (i.e., arms with mean rewards very close to each other) whereas ours pays attention to arms with mean rewards close to the threshold $\rho$.

    \item Moreover, we are able to modify the algorithm for the $\epsilon$-relaxed case. While Track and Stop algorithms can fail for problems with multiple correct solutions if the set of optimal solutions is not convex, as discussed by \citet{degenne2019multcorr}, we demonstrate that the set of optimal proportions is convex and hencwe we can guarantee asymptotic optimality.
    
\end{itemize}

\subsection{Related work}

In the context of the combinatorial pure exploration (CPE) problem, the learner's goal is to explore a set of arms to identify the optimal member of a decision class.  \citet{chen2013combinatorialbandit} introduced general algorithms to address CPE in both fixed confidence and fixed budget settings. \citet{locatelli2016thresbandit} provided a formal definition of the Thresholding Bandit Problem (TBP), along with a summary of lower and upper bounds in the fixed confidence and budget settings. Notably, they presented an algorithm that matches the lower bound of the error probability up to multiplicative constants, and additive $\log (TK)$ terms, in the exponential. 

In the realm of fixed confidence, \citet{kano2018thresbandit} formulated the Good Arm Identification (GAI) problem, where the learner selects good arms, defined as arms with mean rewards above a specified threshold, one by one.

\citet{soare2014bailinear} was the first to introduce Best-Arm Identification (BAI) in the stochastic linear setting and derived an instance-specific information-theoretical lower bound for the expected sample complexity. They offered static and adaptive algorithms inspired by G-optimal experimental design. \citet{tao2018bestarm} introduced an elimination-based algorithm that improved upon the previous upper bound established by \citet{soare2014bailinear}. In \citet{xu2018fully}, they designed the first gap-based algorithm for BAI in the linear case, drawing inspiration from UGapE \cite{gabillon2012bestarm}.

Furthermore, \citet{jedra2020optbailinear} presented an asymptotically optimal algorithm based on the classic Track and Stop algorithm for BAI in the Multi-Armed Bandit (MAB) setting \cite{kaufmann2014bai}. Meanwhile, \citet{degenne2020pureexp} introduced a general pure exploration formulation for the linear case, focusing on single answers. They provided the asymptotic lower bound for sample complexity and an asymptotically optimal algorithm. Their approach treats the problem as a two-player zero-sum game, with the algorithm solving a saddle point to approximate the optimal proportions determined by the problem's complexity.

\section{Problem Formulation}
\label{promblemform}

Consider a set of arms $\mathcal{A} = {1, 2, 3, \ldots, K}$, each associated with a feature vector $x_a \in \mathcal{X} = \{x_1, x_2, \ldots, x_K\} \subset \mathbb{R}^d$, $\spn(\mathcal{X}) = \mathbb{R}^d$, and an unknown parameter vector $\theta \in \mathbb{R}^d$. We assume that $||x_a|| < L$ for all $a \in \mathcal{A}$, where $L > 0$. In round $t \geq 1$, the decision maker selects arm $a_t = a \in \mathcal{A}$ and she observes reward $r_t = x_{a}^T \theta + \eta_t$, where $\eta_t \sim N(0, \sigma^2)$ and $\eta_t$ is independent of the past observations. We denote the history as $\mathcal{H}_t = (a_1, x_{a_1}, r_1, \dots, a_{t-1}, x_{a_{t-1}}, r_{t-1})$ and the filtration generated for this history as $\mathcal{F}_t = \sigma(\mathcal{H}_t)$. 

A pure exploration algorithm comprises a sampling rule $\pi = \{\pi_t \}$, a stopping rule $\tau$ and a decision rule $\phi$. The sampling rule determines the action taken in round $t$ and can be randomized. Therefore, if we denote $\mathcal{D}(\mathcal{A})$ the set of all the distributions
over $\mathcal{A}$, each $\pi_t$ is a mapping from history $\mathcal{H}_t$ to $\mathcal{D}(\mathcal{A})$ , i.e. $\pi_t$ is $\mathcal{F}_t$ measurable and $\pi_t(\mathcal{H}_t) = \pi_t \in  \mathcal{D}(\mathcal{A})$. The stopping rule $\tau$ is the time when the learner stops sampling and makes the decision. Thus, $\tau$ has to be a stopping time with respect to ${\mathcal{F}_t}$. The decision rule $\phi$ is $\mathcal{F}_{\tau}$ measurable and it is a subset of $\mathcal{A}$ according with some optimality criteria. 

In our case, given a threshold $\rho \in \mathbb{R}^d$ and a fixed confidence $\delta \in (0, 1)$, the objective of the learner is to find all the arms with mean reward above $\rho$ in the least time possible with probability at least $1-\delta$. Formally, let $\Pi_{\rho}(\theta) = \{ a \in \mathcal{A} : x_a^T \theta > \rho \}$. We will omit the subscript $\rho$ when it is clear $\Pi_{\rho}(\theta) = \Pi(\theta)$.

\begin{definition}\label{def:def_de_corr} We say that the tuple $(\pi, \tau, \phi)$ is $\delta$-correct if:

$$\mathbb{P}[\Pi_{\rho}(\theta) = \phi] \geq 1 - \delta$$

and 

$$\mathbb{P}[\tau < \infty] = 1$$

\end{definition} 

We can also introduce a relaxation of the threshold condition, permitting the misidentification of arms within an $\epsilon$ radius, where $\epsilon > 0$ , around the threshold.

We can also relax the threshold condition, allowing arms to be misidentified within an $\epsilon$ radius, where $\epsilon > 0$, around the threshold.

\begin{definition}\label{def:def_d_corr}  We say that the tuple $(\pi, \tau, \phi)$ is $(\epsilon, \delta)$-correct if:

$$\mathbb{P}[\Pi_{\rho+\epsilon}(\theta) \subseteq \phi \land \Pi_{\rho-\epsilon}(\theta)^{\mathsf{c}} \subseteq \phi^{\mathsf{c}}] \geq 1 - \delta$$

and 

$$\mathbb{P}[\tau < \infty] = 1$$

\end{definition} 

As we will see, relaxing the threshold criteria reduces the complexity of the problem. We will use the natural choice for the decision rule $\phi = \Pi(\widehat{\theta})$, where $\widehat{\theta} = \widehat{\theta}_{\tau}$ is the least squares estimator of $\theta$ at the stopping time.

\section{Sample Complexity Lower Bound}
\label{sclb}

\subsection{Geometric Intuition}\label{sec:sclb:geom_int}

\textbf{Stopping rule}

\begin{figure}[ht]
\vskip 0.2in
\begin{center}
\centerline{\includegraphics[width=.5\columnwidth]{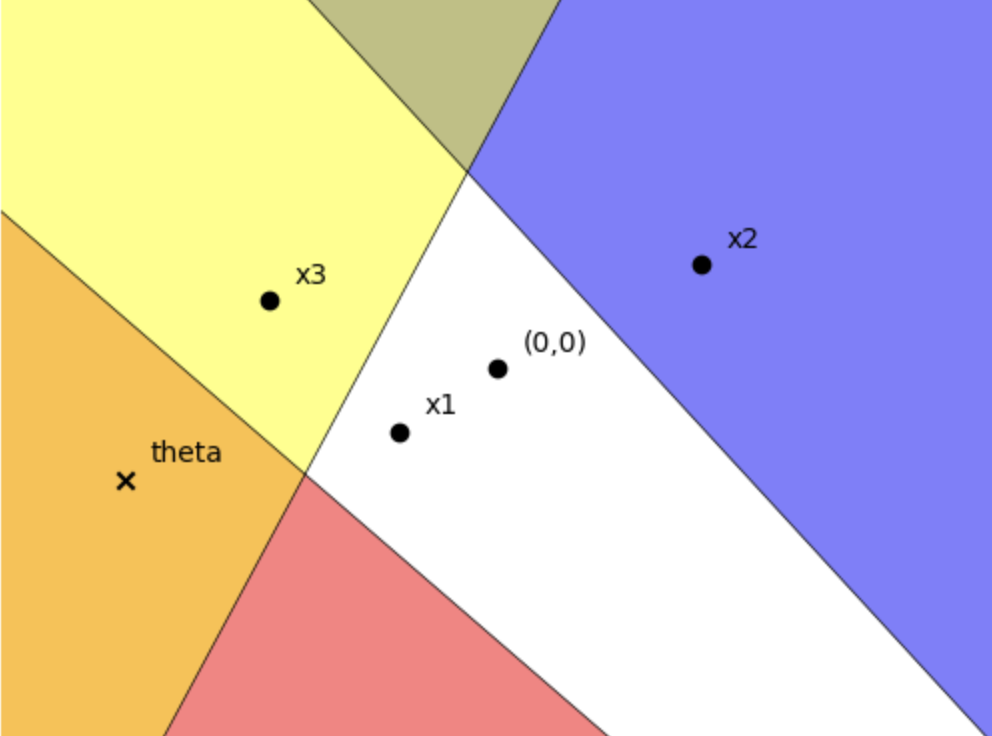}}
\caption{The halved planes represent the value of $\theta$ that makes the corresponding point to be above the threshold $\rho$. The red on corresponds to $x_1$, the blue one to $x_2$ and the yellow one to $x_3$.}
\label{geom_intuition}
\end{center}
\vskip -0.2in
\end{figure}

For the sake of gaining intuition about how well an algorithm can perform, we assume we can use the knowledge of the true $\theta^*$ to construct confidence set and then define the sample and stopping rules. This will show that even an ``oracle" strategy needs a minimum sample size to be confident enough about the answer. Furthermore, it will help to understand how TBP it is different from the BAI in the linear case.

First, following \citet{soare2014bailinear}, let $\mathcal{C}(x) = \{ \theta \in \mathbb{R}^d: x^T \theta > \rho\}$ the set of parameters $\theta$ that makes $x$ have a mean reward above $\rho$. This sets are open half spaces. 

In contrast to \citet{soare2014bailinear}, our definition of $\mathcal{C}(x)$ does not form a partition of $\mathbb{R}^d$. Instead, the true parameter $\theta^{*}$ resides at the intersection of $\mathcal{C}(x)$ for all the "good arms", i.e. $\theta^* \in \cap_{x \in \Pi(\theta^*)} \mathcal{C}(x) = \tilde{\mathcal{C}} $. Furthermore, $\theta^* \in \cap_{x \in \Pi(\theta^*)^{\mathsf{c}}} \mathcal{C}(x)^{\mathsf{c}} = \underaccent{\tilde}{\mathcal{C}}$. Note that all $\theta \in \tilde{\mathcal{C}} \cap \underaccent{\tilde}{\mathcal{C}}$ have the same $\Pi(\theta) = \Pi(\theta^*)$. 

As in \citet{soare2014bailinear}, we assume that for any allocation $\mathbf{x}_n = (x_1, x_2, \dots, x_n)$, it is possible to construct a confidence set $\mathcal{S}^*\left(\mathbf{x}_n\right) \subseteq \mathbb{R}^d$ such that $\theta^* \in \mathcal{S}^*\left(\mathbf{x}_n\right)$ and the (random) OLS estimate $\hat{\theta}_n$ belongs to $\mathcal{S}^*\left(\mathbf{x}_n\right)$ with high probability, i.e., $\mathbb{P}(\hat{\theta}_n \in \mathcal{S}^*(\mathbf{x}_n)) \geq 1-\delta$. As a result, the oracle stopping criterion simply checks whether the
confidence set $\mathcal{S}^*\left(\mathbf{x}_n\right)$ is contained in both $\tilde{\mathcal{C}}$ and $\underaccent{\tilde}{\mathcal{C}}$. That means, for all the probable values of $\widehat{\theta}_n$, there is no ambiguity about the arms that are above/below the threshold.

\begin{proposition}\label{prop:propgeom}

    Let $\mathbf{x}_n$ be an allocation such that $\mathcal{S}^*(\mathbf{x}_n) \subseteq \tilde{\mathcal{C}} \cap \underaccent{\tilde}{\mathcal{C}}$, then 
    
    $$\mathbb{P}[\Pi(\hat{\theta}_n) = \Pi(\theta^*)] \geq 1 - \delta$$

\end{proposition}

\begin{proof}
    The proof follows immediately because of the fact that $\hat{\theta}_n \in \mathcal{S}^*\left(\mathbf{x}_n\right)$ with probability at least $1-\delta$, then $\hat{\theta}_n \in \tilde{\mathcal{C}} \cap \underaccent{\tilde}{\mathcal{C}}$ and this implies that $\Pi(\hat{\theta}_n) = \Pi(\theta^*)$.
\end{proof}
\textbf{Sampling rule}

From the previous arguments,  it is evident that the goal of the optimal strategy should be to shrink the confidence set $\mathcal{S}^*(\mathbf{x}_n)$ such that $\mathcal{S}^*(\mathbf{x}_n) \subseteq \tilde{\mathcal{C}} \cap \underaccent{\tilde}{\mathcal{C}}$ as quick as possible. Formally, $\mathcal{S}^*(\mathbf{x}_n) \subseteq \tilde{\mathcal{C}} \cap \underaccent{\tilde}{\mathcal{C}}$ is equivalent to

$$
\begin{aligned}
\forall \theta \in \mathcal{S}^*(\mathbf{x}_n), \forall a \in \Pi(\theta^{*}), x_a^{\top}\theta > \rho \text{ } \land \\
\forall a \in \Pi(\theta^{*})^{\mathsf{c}}, x_a^{\top}\theta < \rho
\end{aligned}
$$

$$
\begin{aligned}
\forall \theta \in \mathcal{S}^*(\mathbf{x}_n), \forall a \in \Pi(\theta^{*}), x_a^{\top}(\theta^{*} - \theta) < \Delta(x_a) \text{ } \land \\
\forall a \in \Pi(\theta^{*})^{\mathsf{c}}, x_a^{\top}(\theta^{*} - \theta) > - \Delta(x_a)
\end{aligned}
$$

$$
\forall \theta \in \mathcal{S}^*(\mathbf{x}_n), \forall a \in \mathcal{A}, |x_a^{\top}(\theta^{*} - \theta)| < \Delta(x_a)
$$

where $\Delta(x_a) = |x_a^{\top}\theta^* - \rho|$. Then we can define our confidence set

$$
\begin{aligned}
\mathcal{S}^*(\mathbf{x}_n) = \{ \theta \in \mathbb{R}^d, x \in \mathcal{X},
\frac{|x^{\top}(\theta^{*} - \theta)|}{||x||_{A_{\mathbf{x}_n}^{-1}}} < c \sqrt{\beta (t, \delta)} \}
\end{aligned}
$$

Where $A_{\mathbf{x}_n} = \sum_{s=1}^n x_s x_s^{\top}$ and for some $\beta (t, \delta)$. Thus the stopping condition is equivalent to

$$ ||x_a||_{A_{\mathbf{x}_n}^{-1}} \sqrt{\beta (t, \delta)} \leq \Delta(x_a)$$

For all $a \in \mathcal{A}$. From this condition, the oracle allocation strategy simply follows

$$
\begin{aligned}
\mathbf{x}_n^{*} &= \arg \min _{\mathbf{x}_n} \max _{a \in \mathcal{A}} \frac{ ||x_a||_{A_{\mathbf{x}_n}^{-1}} \sqrt{\beta (t, \delta)}}{\Delta(x_a)} \\ 
&= \arg \min _{\mathbf{x}_n} \max _{a \in \mathcal{A}} \frac{||x_a||_{A_{\mathbf{x}_n}^{-1}}}{\Delta(x_a)} 
\end{aligned}
$$

\textbf{Sample complexity}

Now we can define the sample complexity as the minimum number of samples we need to meet the stopping time condition.

$$N^* = \min \left\{n \in \mathbb{N},  \min _{\mathbf{x}_n} \max _{a \in \mathcal{A}} \frac{ ||x_a||_{A_{\mathbf{x}_n}^{-1}} \sqrt{\beta (t, \delta)}}{\Delta(x_a)} < 1 \right\}$$

Note that we can formulate this problem as an optimal design (soft allocation) replacing $A_{\mathbf{x}_n} \approx n A_{\lambda} =  n \sum_{a \in \cA} \lambda_a x_a x_a^{\top}$ with $\lambda \in \Lambda = \{ \lambda \in [0,1]^K : \sum_{k=1}^K \lambda_k = 1 \}$. Then, if we denote the problem complexity

$$T_{\theta}^* = \min _{ \lambda \in \Lambda} \max _{a \in \mathcal{A}} \frac{ ||x_a||_{A_{\lambda}^{-1}}^2}{\Delta(x_a)^2}$$

Then, we can choose $\beta (t, \delta) = \cO (\log (\frac{1}{\delta}))$ such that $\hat{\theta} \in \mathcal{S}^*(\mathbf{x}_n)$ with probability at least $1-\delta$ \cite{abbasiyadkori2011linearbandit}

$$N^* = T^* \cO (\log (\frac{1}{\delta}))$$

\subsection{Lower bound}
The geometric intuition provides valuable insights into the sample complexity lower bound and problem complexity. As we will see, we can leverage standard BAI concepts to establish a non-asymptotic lower bound. Similar to the linear case in BAI, the lower bound for the sample complexity is determined by the optimization problem involving the ratio of the norm induced by the inverse of the design $A_{\lambda}$ and the gaps $\Delta(x)$. In our setting the gaps are defined as the distance between the mean rewards and the threshold. 

It's worth noting that \citet{degenne2020pureexp} established a general asymptotic lower bound for pure exploration problems in the context of stochastic linear bandits, and our lower bound aligns with their findings. We extend this result to a more general setting for thresholding bandits, encompassing the relaxed version of the problem outlined in Definition \ref{def:def_de_corr}.

\begin{theorem}\label{thm:thmsclb}
For any linear bandit environment $\nu=(\cA, \cX, \theta)$, the sample complexity $\tau$ of any $\delta$-$\epsilon$-correct tuple strategy satisfies:
$$
\mathbb{E}[\tau] \geq \sigma^2 \log (1 / 2 \delta) T_\theta^{\star},
$$

where $T_\theta^{\star} = \min _{\lambda \in \Lambda} \max _{a \in \mathcal{A}}\frac{2\|x_a\|_{A_\lambda^{-1}}^2}{\Delta(x_a)^2}$, $\Lambda = \{ \lambda \in [0,1]^K : \sum_{k=1}^K \lambda_k = 1 \}$, $A_\lambda=\sum_{a \in \mathcal{A}} \lambda_a x_a x_a^{\top}$ captures the sample repartition over the arms $x$, as given by the design $\lambda$, $\Delta(x_a) = |x_a^{\top} \theta - \rho| + \epsilon$ is the gap.

\end{theorem}

\begin{proof}

We will follow the usual approach for lower bound for the sample complexity introduce in \citep{kaufmann2014bai}. Denote by $d(x, y)=x \log (x / y)+(1-x) \log ((1-x) /(1-y))$ the binary relative entropy.
Let $\nu'=(\cA, \cX, \theta'), \theta' \in \mathbb{R}^d$ a different bandit enviroment such that $\exists$ $a \in \mathcal{A}$ for which $x_a^{\top} \theta > \rho + \epsilon$ and $x_a^{\top} \theta' < \rho - \epsilon$ or $x_a^{\top} \theta < \rho - \epsilon$ and $x_a^{\top} \theta' > \rho + \epsilon$, i.e. $\Pi_{\rho + \epsilon} (\theta) \nsubseteq \Pi_{\rho - \epsilon} (\theta')$ or $\Pi_{\rho + \epsilon} (\theta') \nsubseteq \Pi_{\rho - \epsilon} (\theta)$.


Then consider the event 
$A = \{ \Pi_{\rho+\epsilon}(\theta) \subseteq \Pi_{\rho}(\widehat{\theta}) \land \Pi_{\rho-\epsilon}(\theta)^{\mathsf{c}} \subseteq \Pi_{\rho}(\widehat{\theta})^{\mathsf{c}}\}$, because the tuples are $\delta$-$\epsilon$-correct we have:

$$\mathbb{P}_{\theta}[A] \geq 1 - \delta$$

$$\mathbb{P}_{\theta'}[A] < \delta$$

$$d\left(\mathbb{P}_\theta[A], \mathbb{P}_{\theta'}[A]\right) \geq \log (1 / 2 \delta)$$

Using Lemma 1 from \citep{kaufmann2014bai} we have

$$\sum_{i=1}^K \mathbb{E}_\nu\left[N_{a}\right] KL\left(\nu_a, \nu'_a\right) \geq d\left(\mathbb{P}_\theta[A], \mathbb{P}_{\theta'}[A]\right)$$

$$\mathbb{E}_\nu[\tau]
\sum_{i=1}^K \frac{\mathbb{E}_\nu\left[N_{a}\right]}{\mathbb{E}_\nu[\tau] } \frac{(x_a^{\top} (\theta - \theta'))^2}{2\sigma^2} \geq \log (1 / 2 \delta)$$

$$ \frac{\mathbb{E}_\nu[\tau]}{2} (\theta - \theta')^{\top} \left(
\sum_{i=1}^K \lambda_a (x_a x_a^{\top})\right) (\theta - \theta') \geq 2\sigma^2 \log (1 / 2 \delta)$$

$$ \mathbb{E}_\nu[\tau] \geq 2\sigma^2\log (1 / 2 \delta) \frac{1}{(\theta - \theta')^{\top} A_{\lambda} (\theta - \theta')}$$

Let's consider $\varepsilon = \theta - \theta'$ that allows to maximize the lower bound. Then, the condition that the two environments have different set of good arms is equivalent to:

$$
\begin{aligned}
x_a^{\top}\theta^{\prime} <\rho - \epsilon \Leftrightarrow \\
x_a^{\top} \theta-x_a^{\top} \theta^{\prime}>x_a^{\top} \theta -\rho + \epsilon \Leftrightarrow \\
x_a^{\top} \varepsilon>x_a^{\top} \theta -\rho + \epsilon \Leftrightarrow \\ 
x_a^{\top} \varepsilon > \Delta_\theta(x_a)
\end{aligned}
$$

For some $a \in \Pi_{\rho + \epsilon}(\theta)$ or

$$
\begin{aligned}
x_a^{\top}\theta^{\prime} >\rho +\epsilon \Leftrightarrow  \\
x_a^{\top}\theta^{\prime} - x_a^{\top} \theta> \rho - x_a^{\top} \theta + \epsilon \Leftrightarrow \\ 
-x_a^{\top} \varepsilon>\rho - x_a^{\top} \theta - \epsilon \Leftrightarrow \\ 
-x_a^{\top} \varepsilon > \Delta_\theta(x_a)
\end{aligned}
$$

For some $a \in \Pi_{\rho - \epsilon}(\theta)^{\mathsf{c}}$. The smallest $\varepsilon$ satisfying this condition is given by the solution of the following (relaxed) minimization problem. Denote $\cA_{\rho, \epsilon}(\theta) = \Pi_{\rho + \epsilon}(\theta) \cup \Pi_{\rho - \epsilon}^{\mathsf{c}}(\theta)$, the arms that are outside $\epsilon$ radius interval of $\rho$.

$$
\begin{array}{ll} & \min \varepsilon^{\top} A_\lambda \varepsilon \\ 
\text { s.t. } & \exists a \in  \cA_{\rho, \epsilon}(\theta), |x_a^{\top} \varepsilon |\geq \Delta_\theta(x_a)+\alpha .
\end{array}
$$

Now we can use the same approach as in Theorem 3.1. \citep{soare2015bailinear}.

$$\mathbb{E}[\tau] \geq 2 \sigma^2 \log (1 / 2 \delta) \max _{a \in \cA_{\rho, \epsilon}(\theta)} \frac{\|x_a\|_{A_\lambda^{-1}}^2}{\Delta(x_a)^2}$$

Then 

$$\mathbb{E}[\tau] \geq 2 \sigma^2 \log (1 / 2 \delta) \min _{\lambda \in \Lambda} \max _{a \in \cA_{\rho, \epsilon}(\theta)} \frac{\|x_a\|_{A_\lambda^{-1}}^2}{\Delta(x_a)^2}$$

\end{proof}

\begin{corollary}
For any linear bandit environment $\nu=(\cA, \cX, \theta)$, the sample complexity $\tau$ of any $\delta$-correct tuple strategy satisfies:
$$
\mathbb{E}[\tau] \geq \sigma^2 \log (1 / 2 \delta) T_\theta^{\star},
$$

where $T_\theta^{\star} = \min _{\lambda \in \Lambda} \max _{a \in \mathcal{A}}\frac{2\|x_a\|_{A_\lambda^{-1}}^2}{\Delta(x_a)^2}$ and $\Delta(x_a) = |x_a^{\top} \theta - \rho|$ is the gap.

\end{corollary}

\section{Algorithm}
\label{algo}

The goal of this section is to introduce a method that achieves asymptotic optimality. We will adopt the conventional Track and Stop strategy for BAI, initially introduced by \citet{kaufmann2014bai}. The primary aim of this strategy is to estimate and align with the optimal proportions that lead to asymptotically optimal sample complexity. It is important to note that in the relaxed setting, the issue of multiple correct answers can arise \citep{degenne2019multcorr}. However, we demonstrate that for this setting, the set of optimal proportions $C(\theta)^{\star}  \subseteq \Lambda$ is also convex. As a result, we ensure that the Track and Stop algorithm will converge within the convex hull of optimal solutions, maintaining optimality.

The primary distinction between the stochastic linear bandits case and the traditional Multi-Armed Bandit (MAB) lies in how to compute the optimal proportions. These proportions are dictated by the problem's complexity. Interestingly, in the case of linear rewards, this strategy also proves effective for TBP. 

Furthermore, we can take advantage of the advancements in lazy updates as introduced by \citet{jedra2020optbailinear}. These updates help alleviate the computational burden of recalculating the optimal proportions at each iteration.

\subsection{Least-squares estimator}

Our algorithm employs the ordinary least-squares (OLS) estimator. In order to apply the concentration inequalities from \citet{abbasiyadkori2011linearbandit} we need to establish a modified version. Instead of relying on the regularized least-squares estimator, we enforce exploration, as demonstrated by \citet{jedra2020optbailinear}. It appears that initializing the OLS estimator with $d$ linearly independent arms is sufficient to utilize the confidence ellipsoid from the theorem in \citet{abbasiyadkori2011linearbandit}, albeit with a modification in the upper bound. 

Let $V_0 = \sum_{s=1}^{d} x_s x_s^{\top}$, $V_t = \sum_{s=1}^{t} x_s x_s^{\top}$, $S_t = \sum_{s=1}^{t} \eta_s x_s$, then $\hat{\theta}_t = \theta +  V_t^{-1} S_t$ is the OLS estimator with the observations until time $t$. We have the following Lemma 

\begin{lemma}\label{lemma:lemma_conf_elps}
Let $\left\{F_t\right\}_{t=0}^{\infty}$ be a filtration. Let $\left\{r_t\right\}_{t=0}^{\infty}$ be a real-valued stochastic process such that $r_t$ is $F_t$-measurable, $F_{t-1}$-conditionally
$R$-sub-Gaussian for some $R>0$. Let $\left\{x_t\right\}_{t=0}^{\infty}$ be an $\bbR^d$-valued stochastic process such that $x_t$ is $F_t$-measurable. Assume that exists $L >0$ such that $\left\| x_t\right\| < L$ for all $t \geq 1$, $V_0 = \sum_{s=1}^{d} x_s x_s^{\top}$ is a positive definite matrix, i.e. $\lambda_{\min}(V_0) > 0$ and  $\|S_0\|^2< C$ with $S_0 = \sum_{s=1}^{d} \eta_s x_s$ for some $C > 0$. Then, for any $\delta > 0$, with probability at least $1-\delta$, for all $t \geq d$,

$$c(t, \delta) = 2 \log \left( \frac{1}{\delta} \right) + d\log \left(\frac{tL^2}{d\lambda_{\min}(V_0)} \right) + \frac{C}{\lambda_{\min}(V_0)}$$

$$ 
\begin{aligned}
\left\| \hat{\theta}_t - \theta \right\|_{V_t} 
&\leq R \sqrt{\beta(t, \delta)} 
\end{aligned}
$$

And then for all $x \in \cX$

$$ 
\begin{aligned}
|x^{\top}\theta - x^{\top}\hat{\theta}_t| &\leq \left\| x \right\|_{V_t^{-1}} R \sqrt{\beta(t, \delta)}
\end{aligned}
$$

\end{lemma}

\subsection{Sampling Rule}
We will aim to match the optimal proportions from the optimization problem defined in the lower bound analysis. Denote $\cA_{\rho, \epsilon}(\theta) = \Pi_{\rho + \epsilon}(\theta) \cup \Pi_{\rho - \epsilon}^{\mathsf{c}}(\theta)$

$$
\psi(\theta, \lambda)= 
\begin{cases}
\min _{a \in \cA_{\rho, \epsilon}(\theta)} \frac{\Delta(x_a)^2}{2\|x_a\|_{A_\lambda^{-1}}^2} & \text { if } A_{\lambda} \succ 0, \\ 
0 & \text { otherwise. }
\end{cases}
$$

We prove that $\psi(\theta, \lambda)$ is equivalent to:

$$\psi(\theta, \lambda)=\min _{\left\{\theta': \exists a \in  \cA_{\rho, \epsilon}(\theta), |x_a^{\top} (\theta-\theta') |\geq \Delta_\theta(x_a) \right\}} \|\theta-\theta'\|_{A_\lambda}^2$$

(see proof of Theorem \ref{thm:thmsclb}). To guarantee we can solve this optimization problem and the solution will produce an strategy with optimal sample complexity we need two lemmas. These lemmas are proved by \citet{jedra2020optbailinear} for BAI, we will adapt them to TBP.

\begin{lemma} \label{lemma:lemma_contpsi}
$\psi$ is continuous in both $\theta$ and $\lambda$, and $\lambda \mapsto \psi(\theta, \lambda)$ attains its maximum in $\Lambda$ at a point $\lambda_\theta^{\star}$ such that $\sum_{a \in \mathcal{A}}\left(\lambda_\theta^{\star}\right)_a x_a x_a^{\top}$ is invertible.

\end{lemma}

\begin{lemma}\label{lemma:lemma_maxthm}
(Maximum theorem) Let $\theta \in \mathbb{R}^d$. Define $\psi^*(\theta)=\max _{\lambda \in \Lambda} \psi(\theta, \lambda)$ and $C^{\star}(\theta)=\arg \max _{w \in \Lambda} \psi(\theta, \lambda)$. Then $\psi^{\star}$ is continuous at $\theta$, and $C^{\star}(\theta)$ is convex, compact and non-empty. Furthermore, we have for any open neighborhood $\mathcal{V}$ of $C^{\star}(\theta)$, there exists an open neighborhood $\mathcal{U}$ of $\theta$, such that for all $\theta^{\prime} \in \mathcal{U}$, we have $C^{\star}\left(\theta^{\prime}\right) \subseteq \mathcal{V}$.

\end{lemma}

As emphasized by \citet{kaufmann2014bai}, in stochastic bandits, directly substituting estimated parameters can be perilous. Therefore, we cannot utilize the optimal proportion $C^{\star}(\hat{\theta})$ directly. Instead, we employ the forced exploration sampling method proposed by \citet{jedra2020optbailinear}.

\begin{lemma} \label{lemma:forced_expl}
(Lemma 5 \citep{jedra2020optbailinear}) Let $\mathcal{A}_0=\left\{a_0(1), \ldots, a_0(d)\right\} \subseteq \mathcal{A}: \lambda_{\min }\left(\sum_{a \in \mathcal{A}_0} x_a x_a^{\top}\right)>0$.
Let $\left(b_t\right)_{t \geq 1}$ be an arbitrary sequence of arms. Furthermore, define for all $t \geq 1, f(t)=c_{\mathcal{A}_0} \sqrt{t}$ where $c_{\mathcal{A}_0}=\frac{1}{\sqrt{d}} \lambda_{\min }\left(\sum_{a \in \mathcal{A}_0} a a^{\top}\right)$. Consider the sampling rule, defined recursively as: $i_0=1$, and for $t \geq 0$, $i_{t+1}=\left(i_t \bmod d\right)+ \mathbbm{1}_{\left\{\lambda_{\min }\left(\sum_{s=1}^t x_s x_s^{\top}\right)<f(t)\right\}}$ and

\begin{align}\label{eq:forced}
a_{t+1}= \begin{cases}
a_0\left(i_t\right) & \text { if } \lambda_{\min }\left(\sum_{s=1}^t x_s x_s^{\top}\right)<f(t), \\ b_t & \text { otherwise. }
\end{cases}
\end{align}

Then for all $t \geq \frac{5 d}{4}+\frac{1}{4 d}+\frac{3}{2}$, we have $\lambda_{\min }\left(\sum_{s=1}^t a_s a_s^{\top}\right) \geq f(t-d-1)$.

\end{lemma}

\begin{lemma}\label{lemma:tracking}
(Lemma 6 \citep{jedra2020optbailinear})
Let $(\lambda(t))_{t \geq 1}$ be a sequence taking values in $\Lambda$, such that there exists a compact, convex and non empty subset $C$ in $\Lambda$, there exists $\varepsilon>0$ and $t_0(\varepsilon) \geq 1$ such that $\forall t \geq t_0, d_{\infty}(\lambda(t), C) \leq \varepsilon$. Define for all $a \in \mathcal{A}, N_a(0)=0$. Consider a sampling rule defined by \ref{eq:forced} and
$$
b_t=\underset{a \in \operatorname{supp}\left(\sum_{s=1}^t \lambda(s)\right)}{\arg \min }\left(N_a(t)-\sum_{s=1}^t \lambda_a(s)\right),
$$
where $N_a(0)=0$ and for $t \geq 0, N_a(t+1)=N_a(t)+\mathbbm{1}_{\left\{a_t=a\right\}}$.
Then there exists $t_1(\varepsilon) \geq t_0(\varepsilon)$ such that $\forall t \geq t_1(\varepsilon), d_{\infty}\left(\left(N_a(t) / t\right)_{a \in \mathcal{A}}, C\right) \leq\left(p_t+d-1\right) \varepsilon$ where $p_t=\left|\operatorname{supp}\left(\sum_{s=1}^t \lambda(s)\right) \backslash \mathcal{A}_0\right| \leq K-d$.
\end{lemma}

In \citet{degenne2020pureexp}, they also attempt to align with the optimal proportions established by an equivalent saddle point. However, they adopt an optimistic approach to deal with the uncertainty associated with $\hat{\theta}$.

We will leverage the Lazy update rule from \citet{jedra2020optbailinear}. Therefore, we need to ask the same conditions to guarantee almost sure asymptotical optimality and optimality in terms of expected sample complexity.

There exists a non-decreasing sequence $(\ell(t))_{t \geq 1}$ of integers with $\ell(1)=1, \ell(t) \leq t-1$ for $t>1$ and $\lim _{t \rightarrow \infty} \ell(t)=\infty$ and such that

\begin{align} \label{eq:lazy_cond_1}
\lim _{t \rightarrow \infty} \min _{s \geq \ell(t)} d_{\infty}\left(\lambda(t), C^{\star}\left(\hat{\theta}_s\right)\right)=0 . \quad \text { a.s.. }
\end{align}

As the authors pointed out, this condition is not hard to meet and it it enough to guarantee almost sure asymptotical optimality (see proof of Theorem \ref{thm:thm_ubsce}). On the other hand, to obtain an algorithm with optimal expected sample complexity, we need to consider the following condition: there exist $\alpha>0$ and a non-decreasing sequence $(\ell(t))_{t \geq 1}$ of integers with $\ell(1)=1, \ell(t) \leq t$ and $\liminf _{t \rightarrow \infty} \ell(t) / t^\gamma>0$ for some $\gamma>0$ and such that

$$\forall \varepsilon>0, \exists h(\varepsilon): \forall t \geq 1$$

\begin{align} \label{eq:lazy_cond_2}
\mathbb{P}\left(\min _{s \geq \ell(t)} d_{\infty}\left(\lambda(t), C^{\star}\left(\hat{\theta}_s\right)\right)>\varepsilon\right) \leq \frac{h(\varepsilon)}{t^{2+\alpha}}
\end{align}

\subsection{Stopping Rule}
As shown in section \ref{sec:sclb:geom_int}, the goal of the algorithm is to stop when the confidence set for the parameter has no ambiguity about the correct answer. This can be expressed as

\begin{align}\label{eq:Zt}
 Z(t) &= \min _{a \in \mathcal{A}} \frac{(|x_a^{\top}\widehat{\theta }- \rho| + \epsilon)^2}{2 x_a^{\top}\left(\sum_{s = 1}^t x_s x_s^{\top}\right)^{-1}x_a }
\end{align}


\begin{align}\label{eq:tau}
 \tau &= \inf \left\{ t \geq 1 : Z(t) > \beta(\delta, t), \sum_{s =1}^t x_s x_s^{\top}  \succeq c I_d  \right\}
\end{align}

\begin{align}\label{eq:beta}
 \beta(\delta, t) &= \sigma^2 c(\delta, t)
\end{align}

As noticed by previous works \cite{kaufmann2014bai, jedra2020optbailinear, degenne2020pureexp}, this can be interpreted as a sequential hypothesis testing that stops when the generalized likelihood ratio exceeds the threshold $\beta(\delta, t)$. Therefore, the algorithm will stop at time $\tau$ and will output the arms above and bellow the threshold:

$$ \left\{\Pi_{\rho}(\hat{\theta}), \Pi_{\rho}^{\mathsf{c}}(\hat{\theta})\right\}$$

\begin{lemma}\label{lemma:lemm_decorr}

Under any sampling rule, we have

$$\mathbb{P}[\Pi_{\rho+\epsilon}(\theta) \subseteq \Pi_{\rho}(\hat{\theta}) \land \Pi_{\rho-\epsilon}(\theta)^{\mathsf{c}} \subseteq \Pi_{\rho}^{\mathsf{c}}(\hat{\theta})^{\mathsf{c}}] \geq 1 - \delta$$

\end{lemma}

\subsection{Sample complexity analysis}

In this section we establish algorithm's asymptotic optimality both with high probability and in expectation.

\begin{theorem}\label{thm:thm_ubsce}

Lazy Track-Threshold-and-Stop satisfies the same sample complexity upper bound

$$ \mathbb{P}[\limsup _{\delta \rightarrow 0} \frac{\tau}{\log \left(\frac{1}{\delta}\right)} \lesssim \sigma^2 T_\theta^{\star}] = 1$$

\end{theorem}

\begin{theorem}\label{thm:thm_ubsce_exp}

Lazy Track-Threshold-and-Stop satisfies the same sample complexity upper bound

$$\limsup _{\delta \rightarrow 0} \frac{\mathbb{E}[\tau]}{\log \left(\frac{1}{\delta}\right)} \lesssim \sigma^2 T_\theta^{\star}$$

\end{theorem}

\begin{algorithm}[H] \caption{Lazy Track-Threshold-and-Stop} \label{algorithm_prob6} \begin{algorithmic}[1] 

\STATE \textbf{Imput}: Arms $\mathcal{A}$, confidence level $\delta$, set $\mathcal{T}$ of lazy update.

\STATE \textbf{Initialization} $t=0$, $i=0$, $A_0=0$, $Z(0)=0$,$N(0)=(N_a(0))_{a \in \mathcal{A}}=0$
\STATE \textbf{while}: $(\lambda_{min}(A_t) < c)$ or $Z(t) < \beta(\delta, t)$ \textbf{do}
\STATE $\quad$ \textbf{if} $\lambda_{min}(A_t) < f(t)$ \textbf{then}
\STATE $\quad$ $\quad$ $a \xleftarrow{} a_0(i+1), i \xleftarrow{} (i + 1 mod d) $
\STATE $\quad$ \textbf{else}
\STATE $\quad$ $\quad$ $a \leftarrow \arg \min _{b \in \operatorname{supp}\left(\sum_{s=1}^t w(s)\right)}\left(N_b(t)-\sum_{s=1}^t w_b(t)\right)$
\STATE $\quad$ \textbf{end if}
\STATE $\quad$ $t \leftarrow t+1$, 
\STATE $\quad$ sample arm $a$ and 
\STATE $\quad$ update $N(t), \hat{\theta}_{t}, Z(t), A_{t} \leftarrow A_{t-1}+a a^{\top}, w(t) \leftarrow w(t-1)$
\STATE $\quad$ \textbf{if} $t \in \mathcal{T}$ \textbf{then}
\STATE $\quad$ $\quad$ $w(t) = \arg \max _{w \in \Lambda} \psi (\widehat{\theta}_t, w)$
\STATE $\quad$ \textbf{end if}
\STATE \textbf{end while}
\STATE \textbf{return} $\Pi(\widehat{\theta}_{\tau}) = \{ a\in\mathcal{A} | x_a^{\top}\widehat{\theta}_{\tau} > \rho\}$
\end{algorithmic} 
\end{algorithm}

\section{Experiments}
\label{exp}
In this section we compare the performance of Lazy TTS with some algorithms from BAI in the linear case: LinGapE \cite{xu2018fully}, G allocation \citep{soare2014bailinear}, RAGE \citep{fiez2019sequential}. To the best of our knowledge, there is only one algorithm that addresses TBP in the linear case proposed by \citep{degenne2020pureexp}. To have a fair comparison with algorithms designed for BAI, we will adapt them to handle the TBP. The main idea is that most of these algorithms are designed to reduce the uncertainty of the gaps between the optimal and the rest of arms $x_{a^*}^{\top}\theta - x_{a}^{\top}\theta$. Instead, we need to reduce the uncertainty of the gaps with respect to the threshold $|x_{a}^{\top}\theta - \rho|$. We implement the sampling rule of each algorithm and use the same stopping rule for all of them. Another modification we need to make is modifying the standard experiments on synthetic data. 

\textbf{Linear bandits: modified benchmark example.} One of the benchmark examples in the linear bandit BAI literature introduced by \citet{soare2014bailinear} is the following. Consider $\cX = \{ e_1, \dots, e_d, x'\} \subseteq \bbR^d$ where $e_i$ is the $i$-standard basis vector, $x' = \cos (\alpha)e_1 + \sin (\alpha)e_2$ with $\alpha$ small,  and $\theta = e_1$ so that $1 = \arg \max _{a \in \cA} x_a^{\top}\theta$. This setting results in a hard problem for BAI because the gap between $e_1$ and $x'$ is small, which implies a large complexity of the problem. For TBP this is not necessarily a hard problem, if we choose a threshold $\rho = 0.5$, all arms would have a gap of almost $0.5$ with respect to the threshold. We proposed a slightly modified version of this setting such that it is hard to distinguish if one of the arms is above or below the threshold. Consider similar setting as before, $\cX = \{ e_1, \dots, e_d, x', x''\} \subseteq \bbR^d$ where $e_i$ is the $i$-standard basis vector, $x' = 0.55(\cos (\alpha)e_1 + \sin (\alpha)e_2)$, $x'' = 0.45(\cos (-\alpha)e_1 + \sin (-\alpha)e_2)$ and $\alpha=0.1$ small. Let $\rho = 5$, $\theta = 10e_1$, such that $e_1^{\top}\theta > \rho$, $x^{\top}\theta < \rho$ for $x\in \cX \setminus \{e_1\}$, and $x'^{\top}\theta, x''^{\top}\theta \approx \rho$.

\textbf{Uniform Distribution on a Sphere.} In this example, $\cX$ is sampled from a unit
sphere of dimension $d = 2$ centered at the origin. We set $\theta$ randomly such that $\|\theta\| = 10$, and $\rho = 0$. To control the complexity across the experiments, we discard arms with mean reward in $(-0.15, 0.15)$.


\section{Summary and discussion}
\label{discuss}

In this work, we establish a finite sample complexity lower bound for TBP in the linear case and present an asymptotically optimal algorithm inspired by BAI. We hope that our work helps to extend TBP to other structured frameworks, such as Generalized Linear Models. Additionally, we are interested in exploring scenarios where we can only sample each arm once. This setting is known as Active Search (AS) or Adaptive Sampling for Discovery (ASD) and has been studied for the regret minimization objective. We believe that by combining both directions, we can develop sampling strategies that accelerate scientific discoveries in situations where there is no a fixed budget, but instead we want to know when to stop sampling.

\begin{table}[t]
\caption{Average samples complexity for different feature space dimensions}
\label{sample-table}
\vskip 0.15in
\begin{center}
\begin{small}
\begin{sc}
\begin{tabular}{ccccccc}
\toprule
$d$ &  random &  lazy TTS (ours)&  LinGapE &  $\mathcal{X}$-static & RAGE & LinGame\\
\midrule
2&	86&	\textbf{45}&	58&	72&	55&	46\\
5&	266&	\textbf{69}&	92&	362&	261&	100\\
10&	1066&	\textbf{136}&	168&	1240&	818&	217\\
20&	3786&	\textbf{353}&	481&	5299&	1938&	597\\
\bottomrule
\end{tabular}
\end{sc}
\end{small}
\end{center}
\vskip -0.1in
\end{table}

\begin{table}[t]
\caption{Average samples complexity for different number of arms}
\label{sample-table}
\vskip 0.15in
\begin{center}
\begin{small}
\begin{sc}
\begin{tabular}{ccccccc}
\toprule
$d$ &  random &  lazy TTS (ours) &  LinGapE &  $\mathcal{X}$-static & RAGE & LinGame\\
\midrule
100  &    926 &    \textbf{198} &   205 &   381 & 353 &   222 \\
200  &   1886 &    390 &   \textbf{389} &   646 & 552 &   484 \\
300  &   2324 &    \textbf{424} &   448 &   681 & 563 &   528 \\
500  &   3226 &    \textbf{690} &   697 &   1144 & 959 &   844 \\
1000 &   4236 &    792 &   \textbf{725} &   1274 & 876 &  1251 \\
\bottomrule
\end{tabular}
\end{sc}
\end{small}
\end{center}
\vskip -0.1in
\end{table}

\bibliographystyle{achemso}
\bibliography{references}  

\newpage
\section{Appendix}

In order to prove Lemma \ref{lemma:lemma_conf_elps} we need some lemmas first.

\begin{lemma} \label{lemma:lemma_self_norm_mart_modif_lem8} (Lemma 8 from \citep{abbasiyadkori2011linearbandit})

Let $\lambda \in \bbR^d$ be arbitrary and consider for any $t \geq 0$

$$M_t^{\lambda} = \exp \left( \sum_{s=1}^r \left[ \frac{\eta_s \lambda^{\top}x_s}{R} - \frac{1}{2} (\lambda^{\top}x_s)^2 \right] \right)$$

Let $\tau$ be a stopping time with respect to filtration ${F_t}_{t=0}^{\infty}$. then $M_{\tau}^{\lambda}$ is almost surely well defined and

$$\bbE[M_t^{\lambda}] \leq 1$$

\end{lemma}

\begin{lemma} \label{lemma:lemma_self_norm_mart_modif_lem9}

Under the same assumptions of \ref{lemma:lemma_conf_elps}, for $\delta > 0$

$$\bbP\left[ \exists t > t_0 : \|S_t\|^2_{V_t^{-1}} \geq R^2\left(2\log \left( \frac{1}{\delta} \right) + \log \left(\frac{\det(V_t)}{\det(V_0)}\right) + \frac{C}{\lambda_{\min}(V_0)}\right) \right] < \delta $$

\end{lemma}

\begin{proof}
    Without loss of generality, assume $R=1$. Let 

    $$
    \begin{aligned}
        \hspace{2cm}     
        V_{t_0, t} &= \sum_{s=t_0+1}^t x_s x_s^{\top} \hspace{2cm}     & S_{t_0, t} &= \sum_{s=t_0+1}^t \eta_s x_s &  \hspace{2cm} \\
    \end{aligned} 
    $$  
and
    $$
    \begin{aligned}
        M_{t}^{\lambda} &= \exp \left( \lambda^{\top} S_{t_0, t_0 + t} - \frac{1}{2} \lambda^{\top} V_{t_0, t_0 + t} \lambda \right) \\
    \end{aligned}
    $$

Note that $V_{t}=V_{t_0}+V_{t_0, t}$ and $S_{t}=S_{t_0}+S_{t_0, t}$ for $t \geq t_0$. By Lemma Lemma 20.2 in \citep{lattimore2018banditalgo}, $M_{t}^\lambda$ is a non-negative supermartingale with $\bbE[M_{0}^\lambda] = 1$. Let $\Lambda \sim N(V_0^{-1}S_0, V_0^{-1})$ which is independent of all the other random variables. Define

$$M_{t} = \bbE[M_{t}^{\Lambda} | F_{\infty}]$$

where $F_{\infty}$ is the tail $\sigma$-algebra of the filtration. We can calculate a more explicit form of $M_{t_0, t}$. Let $f(\lambda)$ be the pdf of $\Lambda$ and $c(P) = \sqrt{(2\pi)^d/\det(P)}$. Then

$$
    \begin{aligned}
        M_{t} &= \int_{\bbR^d} \exp \left( \lambda^{\top} S_{t_0, t} - \frac{1}{2} \lambda^{\top} V_{t_0, t} \lambda\right) f(\lambda) d \lambda\\
        &= \frac{1}{c(V_0)} \int_{\bbR^d} \exp \left( \lambda^{\top} S_{t} - \frac{1}{2} \lambda^{\top} V_{t} \lambda -\lambda^{\top} S_{t_0} + \frac{1}{2} \lambda^{\top} V_{t_0} \lambda \right) \exp \left( -\frac{1}{2} (\lambda - V_0^{-1}S_0)^{\top} V_0 (\lambda - V_0^{-1}S_0) \right)  d \lambda\\
        &= \frac{1}{c(V_0)} \int_{\bbR^d} \exp \left( -\frac{1}{2} \| \lambda - V_t^{-1}S_t\|^2_{V_t} + \frac{1}{2}\|S_t\|^2_{V_t^{-1}} -\frac{1}{2} \| S_0\|_{V_0^{-1}}^2 \right)\\
        &= \frac{1}{c(V_0)} \exp \left(\frac{1}{2}\|S_t\|^2_{V_t^{-1}} -\frac{1}{2} \| S_0\|_{V_0^{-1}}^2 \right) \int_{\bbR^d} \exp \left( -\frac{1}{2} \| \lambda - V_t^{-1}S_t\|^2_{V_t} \right)\\
        &= \frac{c(V_t)}{c(V_0)} \exp \left(\frac{1}{2}\|S_t\|^2_{V_t^{-1}} -\frac{1}{2} \| S_0\|_{V_0^{-1}}^2 \right)\\
        &= \left(\frac{\det(V_0)}{\det(V_t)}\right)^{1/2} \exp \left(\frac{1}{2}\|S_t\|^2_{V_t^{-1}} -\frac{1}{2} \| S_0\|_{V_0^{-1}}^2 \right)
    \end{aligned} 
$$

Note that 

$$
\begin{aligned}
    \| S_0\|_{V_0^{-1}}^2 &=  S_0^{\top} V_0^{-1} S_0\\
    &\leq \frac{1}{\lambda_{\min}(V_0)} \|S_0\|^2 \\
    &\leq \frac{C}{\lambda_{\min}(V_0)}
\end{aligned}
$$

By Lemma 20.3 in \citep{lattimore2018banditalgo} $M_t$ is a non-negative supermartingale and we can apply the maximal inequality (Theorem 3.9 in \citep{lattimore2018banditalgo} )

$$
    \begin{aligned}
        \bbP\left[ \exists t > t_0 : \|S_t\|^2_{V_t^{-1}} \geq 2 \log \left( \frac{1}{\delta} \right) + \log \left(\frac{\det(V_t)}{\det(V_0)}\right) + \frac{C}{\lambda_{\min}(V_0)} \right] &\leq \\
        \bbP\left[ \exists t > t_0 : \|S_t\|^2_{V_t^{-1}} \geq 2 \log \left( \frac{1}{\delta} \right) + \log \left(\frac{\det(V_t)}{\det(V_0)}\right) + \| S_0\|_{V_0^{-1}}^2\right] &= \bbP\left[ \sup_{ t > t_0} \log M_{t} \geq \log \left( \frac{1}{\delta} \right)\right] \\
        &= \bbP\left[ \sup_{ t > t_0} M_{t} \geq  \frac{1}{\delta} \right] \\
        &\leq\bbE[M_{0}]\delta \\
        &\leq \delta
    \end{aligned} 
$$

\end{proof}

\textbf{Lemma \ref{lemma:lemma_conf_elps}}. 
Let $\left\{F_t\right\}_{t=0}^{\infty}$ be a filtration. Let $\left\{r_t\right\}_{t=0}^{\infty}$ be a real-valued stochastic process such that $r_t$ is $F_t$-measurable, $F_{t-1}$-conditionally
$R$-sub-Gaussian for some $R>0$. Let $\left\{x_t\right\}_{t=0}^{\infty}$ be an $\bbR^d$-valued stochastic process such that $x_t$ is $F_t$-measurable. Assume that exists $L >0$ such that $\left\| x_t\right\| < L$ for all $t \geq 1$, $V_0 = \sum_{s=1}^{d} x_s x_s^{\top}$ is a positive definite matrix, i.e. $\lambda_{\min}(V_0) > 0$ and  $\|S_0\|^2< C$ with $S_0 = \sum_{s=1}^{d} \eta_s x_s$ for some $C > 0$. Then, for any $\delta > 0$, with probability at least $1-\delta$, for all $t \geq d$,

$$ 
\begin{aligned}
\left\| \hat{\theta}_t - \theta \right\|_{V_t} &\leq R \sqrt{2 \log \left( \frac{1}{\delta} \right) + d\log \left(\frac{tL^2}{\lambda_{\min}(V_0)} \right) + \frac{C}{\lambda_{\min}(V_0)}} 
\end{aligned}
$$

And then for all $x \in \cX$

$$ 
\begin{aligned}
|x^{\top}\theta - x^{\top}\hat{\theta}_t| &\leq \left\| x \right\|_{V_t^{-1}} R \sqrt{2 \log \left( \frac{1}{\delta} \right) + d\log \left(\frac{tL^2}{\lambda_{\min}(V_0)} \right) + \frac{C}{\lambda_{\min}(V_0)}}
\end{aligned}
$$

\begin{proof}

First note that 

$$
\begin{aligned}
\text{trace}(V_t) &= \sum_{s=1}^t \text{trace}(x_s x_s^{\top}) \\
&= \sum_{s=1}^t \|x_s\|^2\\
&\leq t L^2 
\end{aligned}
$$

Then

$$
\begin{aligned}
\det (V_t) &\leq (\text{trace}(V_t)/d )^d \\
&\leq (t L^2 / d)^d
\end{aligned}
$$

On the other hand

$$
\det (V_0) \geq (\lambda_{\min} (V_0))^d
$$

Thus

$$\log \left(\frac{\det(V_t)}{\det(V_0)}\right) \leq d \log \left(\frac{t L^2}{d\lambda_{\min} (V_0)}\right)$$

Let $x \in \cX$, then using Lemma \ref{lemma:lemma_self_norm_mart_modif_lem9}, we have that with probability at least $1 - \delta$

    $$ 
    \begin{aligned}
    |x^{\top}\theta - x^{\top}\hat{\theta}_t| &\leq |x^{\top} V_t^{-1}S_t|  \\ 
    &\leq \|x\|_{V_{t}^{-1}} \|S_t\|_{V_{t}^{-1}} \\
    &=\left\| x \right\|_{V_t^{-1}} R \sqrt{2 \log \left( \frac{1}{\delta} \right) + d\log \left(\frac{tL^2}{\lambda_{\min}(V_0)} \right) + \frac{C}{\lambda_{\min}(V_0)}}
    \end{aligned}
    $$

In particular for $x = V_t (\theta - \hat{\theta}_t)$

    $$ 
    \begin{aligned}
    |x^{\top}\theta - x^{\top}\hat{\theta}_t|
    &= |(\theta - \hat{\theta}_t)^{\top} V_t (\theta - \hat{\theta}_t)| \\
    &= \left\|\theta - \hat{\theta}_t \right\|_{V_t}^2  \\
    &\leq \left\| V_t (\theta - \hat{\theta}_t) \right\|_{V_t^{-1}} R \sqrt{2 \log \left( \frac{1}{\delta} \right) + d\log \left(\frac{tL^2}{\lambda_{\min}(V_0)} \right) + \frac{C}{\lambda_{\min}(V_0)}}\\
    &= \left\|\theta - \hat{\theta}_t \right\|_{V_t} R \sqrt{2 \log \left( \frac{1}{\delta} \right) + d\log \left(\frac{tL^2}{\lambda_{\min}(V_0)} \right) + \frac{C}{\lambda_{\min}(V_0)}}
    \end{aligned}
    $$

Then

    $$ 
    \begin{aligned}
    \left\|\theta - \hat{\theta}_t \right\|_{V_t}
    &\leq R \sqrt{2 \log \left( \frac{1}{\delta} \right) + d\log \left(\frac{tL^2}{\lambda_{\min}(V_0)} \right) + \frac{C}{\lambda_{\min}(V_0)}}
    \end{aligned}
    $$
    
\end{proof}

\textbf{Lemma \ref{lemma:lemma_contpsi}}. 

We have:
$$
\psi(\theta, \lambda)= 
\begin{cases}
\min _{a \in \cA_{\rho, \epsilon}} \frac{\|x_a\|_{A_\lambda^{-1}}^2}{\Delta(x_a)^2} & \text { if } \sum_{a \in \mathcal{A}} \lambda_a a a^{\top} \succ 0, \\ 
0 & \text { otherwise. }
\end{cases}
$$
In addition, $\psi$ is continuous in both $\theta$ and $\lambda$, and $\lambda \mapsto \psi(\theta, \lambda)$ attains its maximum in $\Lambda$ at a point $\lambda_\theta^{\star}$ such that $\sum_{a \in \mathcal{A}}\left(\lambda_\theta^{\star}\right)_a x_a x_a^{\top}$ is invertible.

\begin{proof}

Let $\theta \in \mathbb{R}^d$ such that $x_a^{\top}\theta \neq \rho$  $\forall a \in \mathcal{A}$ and $\lambda \in \Lambda$. Consider the set of bad parameters with respect to $\theta$, $B(\theta) \subset \mathbb{R}^d$

$$
B(\theta)=\left\{\theta' \in \mathbb{R}^d: \Pi_{\rho + \epsilon}(\theta)  \nsubseteq \Pi_{\rho - \epsilon}(\theta') \lor \Pi_{\rho + \epsilon}(\theta')  \nsubseteq \Pi_{\rho - \epsilon}(\theta) \right\}
$$

and denote
$$
f(\theta, \theta', \lambda)=\frac{1}{2}(\theta-\theta')^{\top}\left(\sum_{a \in \mathcal{A}} \lambda_a x_a x_a^{\top}\right)(\theta-\theta') .
$$
Let $\left(\theta_t, \lambda_t\right)_{t \geq 1}$ be a sequence taking values in $\mathbb{R}^d \times \Lambda$ and converging to $(\theta, \lambda)$. Let $\xi<1 \wedge \min _{a \in \cA_{\rho, \epsilon}} \frac{|x_a^{\top}\theta-\rho|-\epsilon}{2\left\|x_a\right\|}$, and let $t_1 \geq 1$ such that for all $t \geq t_1$ we have $\left\|\left(\theta_t, \lambda_t\right)-(\theta, \lambda)\right\|<\xi$. Now, we can verify that by our choice of $\xi$ it holds that $B\left(\theta_t\right)=B(\theta)$. Let $\theta' \in B(\theta)$, then $\Pi_{\rho + \epsilon}(\theta)  \nsubseteq \Pi_{\rho - \epsilon}(\theta')$ (the case $\Pi_{\rho + \epsilon}(\theta')  \nsubseteq \Pi_{\rho - \epsilon}(\theta)$ is analogous). Then exists $a \in \Pi_{\rho + \epsilon}(\theta)$ such that $x_a^{\top} \theta > \rho + \epsilon$ and $x_a^{\top} \theta' < \rho - \epsilon$. Note that

$$
\begin{aligned}
    x_a^{\top} \theta - x_a^{\top} \theta_t &< \left\| x_a\right\| \left\| \theta - \theta_t \right\| \\ 
    &\leq |x_a^{\top}\theta - \rho| - \epsilon \\
    &= x_a^{\top}\theta - \rho - \epsilon
\end{aligned}
$$

Then $x_a^{\top} \theta_t > \rho + \epsilon  \implies \theta' \in B(\theta_t)$ and $B(\theta) \subseteq B(\theta_t)$. Conversely, let $\theta' \in B(\theta_t)$, following similar ideas

$$
\begin{aligned}
    |x_a^{\top} \theta_t - x_a^{\top} \theta| &< \left\| x_a\right\| \left\| \theta_t - \theta \right\| \\ 
    &\leq \frac{1}{2} ( |x_a^{\top}\theta - \rho| - \epsilon) \\
    &\leq \frac{1}{2} ( |x_a^{\top}\theta- x_a^{\top}\theta_t| + |x_a^{\top}\theta_t- \rho| - \epsilon) \\
\end{aligned}
$$

then

$$
\begin{aligned}
    \frac{1}{2}(x_a^{\top} \theta_t - x_a^{\top} \theta) &\leq \frac{1}{2}|x_a^{\top}\theta_t- x_a^{\top}\theta|\\ 
    &< \frac{1}{2} ( |x_a^{\top}\theta_t- \rho| - \epsilon) \\
    &= \frac{1}{2} ( x_a^{\top}\theta_t- \rho - \epsilon)
\end{aligned}
$$

and 

$$
\begin{aligned}
   x_a^{\top} \theta_t - x_a^{\top} \theta < x_a^{\top}\theta_t- \rho - \epsilon
\end{aligned}
$$

Thus $x_a^{\top} \theta > \rho + \epsilon  \implies \theta' \in B(\theta)$ and $B(\theta) = B(\theta_t)$.
Furthermore, note that $f(\theta, \theta', \lambda)$ is a polynomial in $\theta, \theta', \lambda$, thus it is in in particular continuous in $\theta, \lambda$, and there exists $t_2 \geq 1$ such that for all $t \geq t_2$ and for all $\theta' \in \mathbb{R}^d$, it holds that $\left|f\left(\theta_t, \theta', \lambda_t\right)-f\left(\theta, \theta', \lambda_t\right)\right| \leq \xi f\left(\theta, \theta', \lambda\right)$. Hence, with our choice of $\xi$, we have for all $t \geq t_1 \vee t_2$
$$
\begin{aligned}
\left|\psi(\theta, \lambda)-\psi\left(\theta_t, \lambda_t\right)\right| & =\left|\min _{\theta' \in B(\theta)} f(\theta, \theta', \lambda)-\min _{\theta' \in B(\theta_t)} f\left(\theta_t, \theta', \lambda_t\right)\right| \\
& \leq \xi\left|\min _{\theta' \in B(\theta)} f(\theta, \theta', \lambda)\right| \\
& \leq \xi|\psi(\theta, \lambda)| .
\end{aligned}
$$

Thus $\psi$ is continuos in $\theta, \lambda$.
Now, we know that $\lambda \mapsto \psi(\theta, \lambda)$ is continuous on $\Lambda$, and by compactness of the simplex, the maximum is attained at some $\lambda_\theta^{\star} \in \Lambda$. Furthermore, since $\mathcal{X}$ spans $\mathbb{R}^d$, we may construct an allocation $\tilde{\lambda}$ such that $\sum_{a \in \mathcal{A}} \tilde{\lambda}_a x_a x_a^{\top}$ is a positive definite matrix. In addition, by construction of $B(\theta)$, there exists some $M>0$ such that for all $\theta' \in B(\theta)$ we have $\|\theta-\theta'\|>M$, which implies that $\psi(\theta, \tilde{\lambda}) \geq M^2 \lambda_{\min }\left(\sum_{a \in \mathcal{A}} \tilde{\lambda}_a x_a x_a^{\top}\right)>0$. On the other hand, for any allocation $\lambda \in \Lambda$ such that $\sum_{a \in \mathcal{A}} \lambda_a x_a x_a^{\top}$ is rank deficient, we may find a $\theta' \in B(\theta)$ where $\theta'-\theta$ is in the null space of $\sum_{a \in \mathcal{A}} \lambda_a x_a x_a^{\top}$. Therefore, $\sum_{a \in \mathcal{A}}\left(\lambda_\theta^{\star}\right)_a x_a x_a^{\top}$ is invertible.

\end{proof}

\textbf{Lemma \ref{lemma:lemma_maxthm}}. 
(Maximum theorem) Let $\theta \in \mathbb{R}^d$. Define $\psi^*(\theta)=\max _{\lambda \in \Lambda} \psi(\theta, \lambda)$ and $C^{\star}(\theta)=\arg \max _{w \in \Lambda} \psi(\theta, \lambda)$. Then $\psi^{\star}$ is continuous at $\theta$, and $C^{\star}(\theta)$ is convex, compact and non-empty. Furthermore, we have for any open neighborhood $\mathcal{V}$ of $C^{\star}(\theta)$, there exists an open neighborhood $\mathcal{U}$ of $\theta$, such that for all $\theta^{\prime} \in \mathcal{U}$, we have $C^{\star}\left(\theta^{\prime}\right) \subseteq \mathcal{V}$.

\begin{proof}
    The lemma is a direct consequence of the maximum theorem (a.k.a. Berge's theorem) \citep{sundaram1996first} and only requires that $\psi$ is continuous in $(\theta, \lambda) \in \mathbb{R}^d \times \Lambda$, that $\Lambda$ is compact, convex and non-empty, and that $\psi$ is concave in $\lambda$ for each $\theta^{\prime} \in \mathbb{R}^d$ in an open neighborhood of $\theta$. These requirements hold naturally in our setting: (i) by Lemma \ref{lemma:lemma_contpsi}, we have for all $\theta \in \mathbb{R}^d$ and for any $\lambda \in \Lambda, \psi$ is continuous in ( $\left.\theta, \lambda\right)$; (ii) $\Lambda$ is a non-empty, compact and convex set; (iii) for all $\theta \in \mathbb{R}^d, \lambda \mapsto \psi(\theta, \lambda)$ is concave as it can be expressed as the infimum of linear functions in $\lambda$. Therefore, the maximum theorem applies and we obtain the desired results.
    
\end{proof}

\textbf{Lemma \ref{lemma:lemm_decorr}}. 

Under any sampling rule, we have

$$\mathbb{P}[\tau < \infty, \Pi_{\rho + \epsilon}(\theta) \nsubseteq \Pi_{\rho}(\widehat{\theta}) \lor \Pi_{\rho - \epsilon}(\theta)^{\mathsf{c}} \nsubseteq \Pi_{\rho}(\widehat{\theta})^{\mathsf{c}}] \leq \delta$$

\begin{proof}

Lets consider 

$$\mathcal{E}_1 = \{ \tau < \infty \} = \{ \exists t \geq 1 : Z(t) > \beta(\delta, t) \land  \sum_{s =1}^t x_s x_s^{\top}  \succeq c I_d \} $$

$$\mathcal{E}_2 = \{ \Pi_{\rho + \epsilon}(\theta) \nsubseteq \Pi_{\rho}(\widehat{\theta})  \} = \{\exists \tilde{a} \in \Pi_{\rho + \epsilon}(\theta): x_{\tilde{a}}^{\top} \widehat{\theta} < \rho \}$$

$$\mathcal{E}_3 = \{ \Pi_{\rho - \epsilon}(\theta)^{\mathsf{c}} \nsubseteq \Pi_{\rho}(\widehat{\theta})^{\mathsf{c}}  \} = \{\exists \tilde{a} \in \Pi_{\rho - \epsilon}(\theta)^{\mathsf{c}}: x_{\tilde{a}}^{\top} \widehat{\theta} > \rho \}$$

$$
\begin{aligned}
\mathcal{E}_1 \cap \mathcal{E}_2 & =\left\{  \exists t \geq 1 : Z(t) > \beta(\delta, t) \land  \sum_{s =1}^t x_s x_s^{\top}  \succeq c I_d \land \exists \tilde{a} \in \Pi_{\rho + \epsilon}(\theta): x_{\tilde{a}}^{\top} \widehat{\theta} < \rho \right\} \\
& = \left\{  \exists t \geq 1 : \min _{a \in \mathcal{A}} \frac{(|x_a^{\top}\widehat{\theta} - \rho| + \epsilon)^2}{2x_a^{\top}\left(\sum_{s = 1}^t x_s x_s^{\top}\right)^{-1}x_a } > \beta(\delta, t) \land  \sum_{s =1}^t x_s x_s^{\top}  \succeq c I_d \land \exists \tilde{a} \in \Pi_{\rho + \epsilon}(\theta): x_{\tilde{a}}^{\top} \widehat{\theta} < \rho \right\} \\
& \subseteq \left\{  \exists t \geq 1, \exists \tilde{a} \in \Pi_{\rho + \epsilon}(\theta): \frac{(|x_{\tilde{a}}^{\top}\widehat{\theta} - \rho| + \epsilon)^2}{2x_{\tilde{a}}^{\top}\left(\sum_{s = 1}^t x_s x_s^{\top}\right)^{-1}x_{\tilde{a}} } > \beta(\delta, t) \land  \sum_{s =1}^t x_s x_s^{\top}  \succeq c I_d \land x_a^{\top} \widehat{\theta} < \rho \right\}.\\
& \subseteq \left\{ \exists t \geq 1, \exists \tilde{a} \in \Pi_{\rho + \epsilon}(\theta): \frac{\rho - x_{\tilde{a}}^{\top}\widehat{\theta} + \epsilon}{2} > ||x_{\tilde{a}}||_{A_t^{-1}} \sqrt{\beta(\delta, t)} \land  \sum_{s =1}^t x_s x_s^{\top}  \succeq c I_d \land x_{\tilde{a}}^{\top} \widehat{\theta} < \rho \right\} \\
& \subseteq \left\{ \exists t \geq 1, \exists \tilde{a} \in \Pi_{\rho + \epsilon}(\theta): x_{\tilde{a}}^{\top}\theta - x_{\tilde{a}}^{\top}\widehat{\theta} + \rho - x_{\tilde{a}}^{\top}\theta > ||x_{\tilde{a}}||_{A_t^{-1}} \sqrt{\beta(\delta, t)} \land  \sum_{s =1}^t x_s x_s^{\top}  \succeq c I_d \land x_a^{\top} \widehat{\theta} < \rho \right\} \\
& \subseteq \left\{ \exists t \geq 1, \exists \tilde{a} \in \Pi_{\rho + \epsilon}(\theta): |x_{\tilde{a}}^{\top}\theta - x_{\tilde{a}}^{\top}\widehat{\theta}| > ||x_{\tilde{a}}||_{A_t^{-1}} \sqrt{\beta(\delta, t)} \land  \sum_{s =1}^t x_s x_s^{\top}  \succeq c I_d \right\} \\
& \subseteq \left\{ \exists t \geq 1, \exists \tilde{a} \in \Pi_{\rho + \epsilon}(\theta): |x_a^{\top}\theta - x_a^{\top}\widehat{\theta}| > ||x_a||_{A_t^{-1}} \sqrt{\beta(\delta, t)} \land  \sum_{s =1}^t x_s x_s^{\top}  \succeq c I_d \right\}
.
\end{aligned}
$$

Using similar arguments

$$
\begin{aligned}
\mathcal{E}_1 \cap \mathcal{E}_3 & \subseteq \left\{ \exists t \geq 1, \exists \tilde{a} \in \Pi_{\rho - \epsilon}(\theta)^{\mathsf{c}}: |x_a^{\top}\theta - x_a^{\top}\widehat{\theta}| > ||x_a||_{A_t^{-1}} \sqrt{\beta(\delta, t)} \land  \sum_{s =1}^t x_s x_s^{\top}  \succeq c I_d \right\}
.
\end{aligned}
$$

Then using Lemma \ref{lemma:lemma_conf_elps}

$$
\begin{aligned}
\mathbb{P}[\tau < \infty, \Pi_{\rho + \epsilon}(\theta) \nsubseteq \Pi_{\rho}(\widehat{\theta}) \lor \Pi_{\rho - \epsilon}(\theta)^{\mathsf{c}} \nsubseteq \Pi_{\rho}(\widehat{\theta})^{\mathsf{c}}] &= \mathbb{P}[\mathcal{E}_1 \cap \left(\mathcal{E}_2 \cup \mathcal{E}_3\right) ] \\
&<\delta
\end{aligned}
$$
\end{proof}

Similarly than \citep{jedra2020optbailinear}, a sufficient condition to $\delta$-$\epsilon$-correctness is $\bbP[\tau<\infty]=1$, in that case we have 

$$
\mathbb{P}[\Pi_{\rho + \epsilon}(\theta) \subseteq \Pi_{\rho}(\widehat{\theta}) \land \Pi_{\rho - \epsilon}(\theta)^{\mathsf{c}} \subseteq \Pi_{\rho}(\widehat{\theta})^{\mathsf{c}}] \geq 1-\delta
$$

\textbf{Theorem \ref{thm:thm_ubsce}}

Lazy Track-Threshold-and-Stop satisfies the same sample complexity upper bound

$$ \mathbb{P}[\limsup _{\delta \rightarrow 0} \frac{\tau}{\log \left(\frac{1}{\delta}\right)} \lesssim \sigma^2 T_\theta^{\star}] = 1$$

\begin{proof}

From Lemma \ref{lemma:lemma_contpsi} and Lemma \ref{lemma:lemma_maxthm} we know $\psi(\theta, \lambda)$ is continuous in both $\theta$ and $ \lambda$ and $C^{\star}(\theta)$ is continuos in $\theta$. Note that

$$\mathcal{E} = \left\{ d_{\infty}((N_a(t)/t)_{a \in \mathcal{A}}, C^{\star}(\theta)) \rightarrow 0 \land \widehat{\theta}_t \rightarrow  \theta \right\}$$

holds with probability 1 (Lemma 3, 5 and Proposition 1 in \citep{jedra2020optbailinear}). Let $\xi >0 $, By continuity of $\psi$, there exists an open neighborhood $\mathcal{V}(\xi)$ of $\left\{\theta\right\} \times C^{\star}(\theta)$ such that for all $(\theta', \lambda') \in \mathcal{V}(\xi)$, it holds that

$$\psi(\theta', \lambda') \geq (1-\xi)\psi(\theta, \lambda^{\star})$$

for any $\lambda^{\star} \in C^{\star}(\theta)$. Under $\mathcal{E}$, there exists $t_0 \geq 1$ such that for all $t \geq t_0$ it holds that $(\widehat{\theta}_t, (N_a(t)/t)_{a \in \mathcal{A}}) \in \mathcal{V}(\xi)$, this for all $t \geq t_0$, it follows that

$$\psi(\widehat{\theta}_t, (N_a(t)/t)_{a \in \mathcal{A}}) \geq (1-\xi)\psi(\theta, \lambda^{\star})$$

By Lemma 5 in \citep{jedra2020optbailinear}, there exists $t_1 \geq 1$ such that for all $t \geq t_1$ we have 

$$\sum_{s=1}^t x_s x_s^{\top} \succ cI_d$$

Then we can write
$$
\begin{aligned}
 Z(t) &= \min _{a \in \mathcal{A}} \frac{(|x_a^{\top}\widehat{\theta }_t- \rho| + \epsilon)^2}{2 x_a^{\top}\left(\sum_{s = 1}^t x_s x_s^{\top}\right)^{-1}x_a } \\ 
 &= t\min _{a \in \mathcal{A}} \frac{(|x_a^{\top}\widehat{\theta }_t- \rho| + \epsilon)^2}{2||x_a||_{A_{\lambda_t}^{-1}}} \\ 
 &= t\psi(\widehat{\theta}_t, (N_a(t)/t)_{a \in \mathcal{A}})
\end{aligned}
$$

Hence, under $\mathcal{E}$ and for $t \geq \max \{t_0, t_1 \}$

$$Z(t) \geq t (1-\xi)\psi(\theta, \lambda^{\star})$$

Then

$$
\begin{aligned}
\tau & =\inf \left\{ t \geq 1 : Z(t) > \beta(\delta, t) \land  \sum_{s =1}^t x_s x_s^{\top}  \succeq c I_d \right\} \\
& \leq \max \{t_0, t_1 \} \lor \inf \left\{ t \geq 1 : t (1-\xi)\psi(\theta, \lambda^{\star}) > \beta(\delta, t) \right\} \\
& \leq \max \{t_0, t_1 \} \lor \inf \left\{ t \geq 1 : t (1-\xi)\psi(\theta, \lambda^{\star}) > \sigma^2 \left(2 \log \left( \frac{1}{\delta} \right) + d\log \left(\frac{tL^2}{d\lambda_{\min}(V_0)} \right) + \frac{C}{\lambda_{\min}(V_0)}\right) \right\} \\
& \lesssim \max \left\{t_0, t_1, \frac{1}{1-\xi}T^{*}_{\theta} \log\left(\frac{1}{\delta}\right)\right\}
\end{aligned}
$$

Where the last inequality uses Leamma 8 in \citep{jedra2020optbailinear}. Thus

$$\mathbb{P}\left(\lim \sup _{\delta \rightarrow 0} \frac{\tau_{\delta}}{\log(\frac{1}{\delta})} \leqslant  T^{*}_{\theta} \right) = 1$$

\end{proof}

\textbf{Theorem \ref{thm:thm_ubsce_exp}}

Lazy Track-Threshold-and-Stop satisfies the same sample complexity upper bound

$$\limsup _{\delta \rightarrow 0} \frac{\mathbb{E}[\tau]}{\log \left(\frac{1}{\delta}\right)} \lesssim \sigma^2 T_\theta^{\star}$$

Because the algorithms have the same nature, we will follow the proof in \citep{jedra2020optbailinear} closely. We include it here for completeness.
\begin{proof} Let $\varepsilon > 0$

\textbf{ Step 1.} By continuity of $\psi$ (see Lemma \ref{lemma:lemma_contpsi}), there exists $\xi_1(\varepsilon)>0$ such that for all $\theta^{\prime} \in \mathbb{R}^d$ and $\lambda^{\prime} \in \Lambda$

\begin{align}\label{eq:cont_inequality}
    \left\{\begin{array}{ll}
\left\|\theta^{\prime}-\theta\right\| & \leq \xi_1(\varepsilon) \\
d_{\infty}\left(\lambda^{\prime}, C^{\star}(\theta)\right) & \leq \xi_1(\varepsilon)
\end{array} \Longrightarrow\left|\psi\left(\theta, \lambda^{\star}\right)-\psi\left(\theta^{\prime}, \lambda^{\prime}\right)\right| \leq \varepsilon \psi\left(\theta, \lambda^{\star}\right)=\varepsilon\left(T_\theta^{\star}\right)^{-1}\right.
\end{align}

for any $\lambda^{\star} \in \arg \min _{\lambda \in C^*(\theta)} d_{\infty}\left(\lambda^{\prime}, \lambda\right)$ (we have $\lambda^{\star} \in C^{\star}(\theta)$ ). Furthermore, by the continuity properties of the correspondance $C^{\star}$ (see Lemma \ref{lemma:lemma_maxthm}), there exists $\xi_2(\varepsilon)>0$ such that for all $\theta^{\prime} \in \mathbb{R}^d$
$$
\left\|\theta-\theta^{\prime}\right\| \leq \xi_2(\varepsilon) \Longrightarrow \max _{\lambda^{\prime \prime} \in C^{\star}\left(\theta^{\prime}\right)} d_{\infty}\left(\lambda^{\prime \prime}, C^{\star}(\theta)\right)<\frac{\xi_1(\varepsilon)}{2(K-1)}
$$

Let $\xi(\varepsilon)=\min \left(\xi_1(\varepsilon), \xi_2(\varepsilon)\right)$. In the following, we construct $T_0$, and for each $T \geq T_0$ an event $\mathcal{E}_T$, under which for all $t \geq T$, it holds
$$
\left\|\theta-\hat{\theta}_t\right\| \leq \xi(\varepsilon) \Longrightarrow d_{\infty}\left(\left(N_a(t) / t\right)_{a \in \mathcal{A}}, C^{\star}(\theta)\right) \leq \xi_1(\varepsilon)
$$

Let $T \geq 1$, and define the following events
$$
\begin{aligned}
\mathcal{E}_{1, T} & =\bigcap_{t=\ell(T)}^{\infty}\left\{\left\|\theta-\hat{\theta}_t\right\| \leq \xi(\varepsilon)\right\} \\
\mathcal{E}_{2, T} & =\bigcap_{t=T}^{\infty}\left\{\inf _{s \geq \ell(t)} d_{\infty}\left(\lambda(t), C^{\star}\left(\hat{\theta}_s\right)\right) \leq \frac{\xi_1(\varepsilon)}{4(K-1)}\right\} \\
& \subseteq \bigcap_{t=T}^{\infty}\left\{\exists s \geq \ell(t): d_{\infty}\left(\lambda(t), C^{\star}\left(\hat{\theta}_s\right)\right) \leq \frac{\xi_1(\varepsilon)}{2(K-1)}\right\} .
\end{aligned}
$$

Note that, under the event $\mathcal{E}_{1, T} \cap \mathcal{E}_{2, T}$, we have for all $t \geq T$, there exists $s \geq \ell(t)$ such that
$$
\begin{aligned}
d_{\infty}\left(\lambda(t), C^{\star}(\theta)\right) & \leq d_{\infty}\left(\lambda(t), C^{\star}\left(\hat{\theta}_s\right)\right)+\max _{\lambda^{\prime} \in C^{\star}\left(\hat{\theta}_s\right)} d_{\infty}\left(\lambda^{\prime}, C^{\star}(\theta)\right) \\
& <\frac{\xi_1(\varepsilon)}{2(K-1)}+\frac{\xi_1(\varepsilon)}{2(K-1)}=\frac{\xi_1(\varepsilon)}{K-1}
\end{aligned}
$$

Define $\varepsilon_1=\xi_1(\varepsilon) /(K-1)$. By Lemma \ref{lemma:tracking} (6 of \cite{jedra2020optbailinear}) , there exists $t_1\left(\varepsilon_1\right) \geq T$ such that
$$
d_{\infty}\left(\left(N_a(t) / t\right)_{a \in \mathcal{A}}, C^{\star}(\mu)\right) \leq\left(p_t+d-1\right) \frac{\xi_1(\varepsilon)}{K-1} \leq \xi_1(\varepsilon)
$$
where $p_t=\left|\operatorname{supp}\left(\sum_{s=1}^t \lambda(s)\right) \backslash \mathcal{A}_0\right|$ and more precisely $t_1\left(\varepsilon_1\right)=\max \left\{1 / \varepsilon_1^3, 1 /\left(\varepsilon_1^2 d\right), T / \varepsilon_1^3, 10 / \varepsilon_1\right\}$ (see the proof of Lemma 6 of \cite{jedra2020optbailinear}). Thus for $T \geq \max \left\{10 \varepsilon_1^2, \varepsilon_1 / d, 1\right\}$, we have $t_1\left(\varepsilon_1\right)=\left\lceil T / \varepsilon_1^3\right\rceil$. Hence, defining for all $T \geq \varepsilon_1^{-3}$, the event
$$
\mathcal{E}_T=\mathcal{E}_{1,\left\lceil\varepsilon_1^3 T\right\rceil} \cap \mathcal{E}_{2,\left\lceil\varepsilon_1^3 T\right\rceil}
$$

we have shown that for all $T \geq T_0=\max \left(10 \varepsilon_1^5, \varepsilon_1^4 / d, \varepsilon_1^3, 1 / \varepsilon_1^3\right)$, the following holds

\begin{align} \label{eq:thm2_2}
\forall t \geq T, \quad\left\|\theta-\theta_t\right\| \leq \xi(\varepsilon) \Longrightarrow d_{\infty}\left(\left(N_a(t) / t\right)_{a \in \mathcal{A}}, C^*(\theta)\right) \leq \xi_1(\varepsilon) .    
\end{align}

Finally, combining the implication \ref{eq:thm2_2} with the fact that \ref{eq:cont_inequality} holds under $\mathcal{E}_T$ we conclude that for all $T \geq T_0$, under $\varepsilon_T$ we have

\begin{align}\label{eq:thm_3} \psi\left(\hat{\theta}_t,\left(N_a(t) / t\right)_{a \in \mathcal{A}}\right) \geq(1-\varepsilon) \psi^*(\theta)    
\end{align}

\textbf{Step 2.}: Let $T \geq T_0 \vee T_1$ where $T_1$ is defined as
$$
T_1=\inf \left\{t \in \mathbb{N}^*: \lambda_{\min }\left(\sum_{s=1}^t x_s x_s^{\top}\right) \succeq c I_d\right\} .
$$
where we recall that $c$ is the constant chosen in the stopping rule and is independent of $\delta$. For all $t \geq T_1$ we have

$$
Z(t)=t \psi\left(\hat{\theta}_i,\left(N_a(t) / t\right)_{a \in \mathcal{A}}\right) \text {, }
$$

Thus under the event $\varepsilon_T$, the inequality \ref{eq:thm_3} holds, and for all $t \geq T$ we have
$$
Z(t)>t(1-\varepsilon)\left(T_\theta^*\right)^{-1}
$$

Under the event $\varepsilon_T$, we have
$$
\begin{aligned}
\tau & =\inf \left\{t \in \mathbb{N}^*: Z(t)>\beta(\delta, t) \text { and } \sum_{*=1}^t a_s a_s^{\top} \succeq c I_1\right\} \\
& \leq \inf \{t \geq T: Z(t)>\beta(\delta, t)\} \\
& \leq T \vee \inf \left\{t \in \mathbb{N}^*: t(1-\varepsilon)\left(T_\theta^*\right)^{-1} \geq \beta(\delta, t)\right\} \\
& \leq T \vee \inf \left\{t \in \mathbb{N}^*: t(1-\varepsilon)\left(T_\theta^*\right)^{-1} \geq \sigma^2 \left(2 \log \left( \frac{1}{\delta} \right) + d\log \left(\frac{tL^2}{d\lambda_{\min}(V_0)} \right) + \frac{C}{\lambda_{\min}(V_0)}\right)\right\}
\end{aligned}
$$
Applying Lemma 8 in \citep{jedra2020optbailinear} yields

$$
\inf \left\{t \in \mathbb{N}^*: t(1-\varepsilon)\left(T_\theta^*\right)^{-1} \geq \sigma^2 \left(2 \log \left( \frac{1}{\delta} \right) + d\log \left(\frac{tL^2}{d\lambda_{\min}(V_0)} \right) + \frac{C}{\lambda_{\min}(V_0)}\right)\right\} \leq T_2^*(\delta) \text {, }
$$

where $T_2^*(\delta)=\frac{c_1}{1-\varepsilon} T_\theta^* \log (1 / \delta)+o(\log (1 / \delta))$ for some $0< c_1 \geq \sigma^2$ independent of $\delta$. This means for $T \geq \max \left\{T_0, T_1, T_2^*(\delta)\right\}$, we have shown that

\begin{align}\label{eq:thm_4}
    \mathcal{E}_T \subseteq\{\tau \leq T\}    
\end{align}

Define $T_3^{\star}(\delta)=\max \left\{T_0, T_1, T_2^*(\delta)\right\}$. We may then write for all $T \geq T_3^{\star}(\delta)$
$$
\tau \leq \tau \wedge T_3^*(\delta)+\tau \vee T_3^*(\delta) \leq T_3^*(\delta)+\tau \vee T_3^*(\delta) .
$$

Taking the expectation of the above inequality, and using the set inclusion \ref{eq:thm_4}, we obtain that
$$
\mathbb{E}[\tau] \leq T_3^*(\delta)+\mathbb{E}\left[\tau \vee T_3^*(\delta)\right]
$$

Now we observe that
$$
\begin{aligned}
\mathbb{E}\left[\tau \vee T_3^*(\delta)\right] & =\sum_{T=0}^{\infty} \mathbb{P}\left(\tau \vee T_3^*(\delta)>T\right) \\
& =\sum_{T=T_3^*(\delta)+1}^{\infty} \mathbb{P}\left(\tau \vee T_3^*(\delta)>T\right) \\
& =\sum_{T=T_3^*(\delta)+1}^{\infty} \mathbb{P}(\tau>T) \\
& \leq \sum_{T=T_3^*(\delta)+1}^{\infty} \mathbb{P}\left(\cE_T^{\mathsf{c}}\right) \\
& \leq \sum_{T=T_0 \vee T_1}^{\infty} \mathbb{P}\left(\cE_T^{\mathsf{c}}\right)
\end{aligned}
$$

We have thus shown that

\begin{align} \label{eq:thm3_5}
    \mathbb{E}[\tau] \leq \frac{c_1}{1-\varepsilon} T_\theta^* \log (1 / \delta)+o(\log (1 / \delta))+T_0 \vee T_1+\sum_{T=T_{\mathrm{h}} \vee T_1}^{\infty} \mathbb{P}\left(\mathcal{E}_T^{\mathsf{c}}\right) .    
\end{align}

\textbf{Step 3}: We now show that $\sum_{T=T_0 \vee T_1+1}^{\infty} \mathbb{P}\left(\mathcal{E}_T^c\right)<\infty$ and that it can be upper bounded by a constant independent of $\delta$. To ensure this, we shall see that there is a minimal rate by which the sequence $(\ell(t))_{t \geq \infty}$ must grow. Let $T \geq T_0 \vee T_1$, we have by the union bound
$$
\mathbb{P}\left(\mathcal{E}_T^{\mathsf{c}}\right) \leq \mathbb{P}\left(\mathcal{E}_{1,\left\lceil \varepsilon_1^3 T\right\rceil}^{\mathsf{c}}\right)+\mathbb{P}\left(\mathcal{E}_{2,\left\lceil\varepsilon_1^3 T\right\rceil}^{\mathsf{c}}\right) .
$$

First, using a union bound and the lazy condition \ref{eq:lazy_cond_2}, we observe that there exists $h\left(\frac{\varepsilon_1(c)}{4(K-1)}\right)>0$ and $\alpha>0$ such that
$$
\begin{aligned}
\mathbb{P}\left(\mathcal{E}_{1,\left\lceil\varepsilon_1^3 T\right\rceil}^{\mathsf{c}}\right) & \leq \sum_{t=\left\lceil\varepsilon_1^3 T\right\rceil}^{\infty} \mathbb{P}\left(\inf _{s \geq \ell(t)} d_{\infty}\left(\lambda(t), C^{\star}\left(\hat{\theta}_s\right)\right)>\frac{\xi_1(\varepsilon)}{4(K-1)}\right) \\
& \leq h\left(\frac{\xi_1(\varepsilon)}{4(K-1)}\right) \sum_{t=\left\lceil\varepsilon_1^3 T\right\rceil}^{\infty} \frac{1}{t^{2+\alpha}} \\
& \leq h\left(\frac{\xi_1(\varepsilon)}{4(K-1)}\right) \int_{\left\lceil\varepsilon_1^3 T\right\rceil-1}^{\infty} \frac{1}{t^{2+\alpha} d t} \\
& \leq h\left(\frac{\xi_1(\varepsilon)}{4(K-1)}\right) \frac{1}{(1+\alpha)\left(\left\lceil\varepsilon_1^3 T\right\rceil-1\right)^{1+\alpha}} .
\end{aligned}
$$

This clearly shows that $\sum_{T=T_0 \vee T_1}^{\infty} \mathbb{P}\left(\mathcal{E}_{1,\left\lceil\varepsilon_1^3 T\right\rceil}^{\mathrm{c}}\right)<\infty$.
Second, we observe, using a union bound, Lemma 5 and Lemma 4 from \cite{jedra2020optbailinear}, that there exists strictly positive constants $c_3, c_4$ that are independent of $\varepsilon$ and $T$, and such that
$$
\begin{aligned}
\mathbb{P}\left(\mathcal{E}_{2,\left\lceil\varepsilon_1^3 T\right]}^{\mathrm{c}}\right) & \leq \sum_{t=\ell\left(\left[\epsilon_1^3 T\right]\right)}^{\infty} \mathbb{P}\left(\left\|\theta-\hat{\theta}_t\right\|>\xi(\varepsilon)\right) \\
& \leq c_3 \sum_{t=\ell\left(\left\lceil\varepsilon_1^3 T\right\rceil\right)}^{\infty} t^{d / 4} \exp \left(-c_4 \xi(\varepsilon)^2 \sqrt{t}\right) .
\end{aligned}
$$

For $t$ large enough, the function $t \mapsto t^{d / 4} \exp \left(-c_4 \xi(\varepsilon)^2 \sqrt{t}\right)$ becomes decreasing. Additionally, we have by assumption that $(\ell(t))_{t \geq 1}$ is a non decreasing and that $\lim _{t \rightarrow \infty} \ell(t)=\infty$, thus we may find $T_2>T_0 \vee T_1$ such that for all $T \geq T_2$, the function $t \mapsto t^{d / 4} \exp \left(-c_4 \xi(\varepsilon)^2 \sqrt{t}\right)$ is decreasing on $\left[\ell\left(\varepsilon_1^3 T\right)-1, \infty\right)$. Hence, for $T \geq T_2$, we have
$$
\mathbf{P}\left(\mathcal{E}_{2,\left\lceil\varepsilon_1^3 T\right\rceil}^c\right) \leq c_3 \int_{\ell\left(\left\lceil\varepsilon_1^3 T\right\rceil\right)-1}^{\infty} t^{d / 4} \exp \left(-c_4 \xi(\varepsilon)^2 \sqrt{t}\right) d t .
$$

Furthermore, for some $T_3 \geq T_2$ large enough, we may bound the integral for all $T \geq T_3$ as follows
$$
\int_{\ell\left(\left\lceil\varepsilon_1^3 T\right\rceil\right)-1}^{\infty} t^{d / 4} \exp \left(-c_1 \xi(\varepsilon)^2 \sqrt{t}\right) d t \lesssim \frac{\ell\left(\left(\left\lceil\varepsilon_1^3 T\right\rceil\right)-1\right)^{d / 2+1}}{\xi(\varepsilon)^4 \exp \left(c_4 \xi(\varepsilon)^2 \sqrt{\left.\ell\left(\lceil \varepsilon_1^3 T\right\rceil\right)-1}\right)} .
$$

We spare the details of this derivation as the constants are irrelevant in our analysis. Essentially, the integral can be expressed through the upper incomplete Gamma function which can be upper bounded using some classical inequalities $[23,24]$. We then obtain that for $T \geq T_3$,
$$
\mathbb{P}\left(\mathcal{E}_{2, \left\lceil \varepsilon_1^3 T\right\rceil}^{\mathrm{c}}\right) \lesssim \frac{\ell\left(\left(\left\lceil \varepsilon_1^3 T\right\rceil\right)-1\right)^{d / 2+1}}{\xi(\varepsilon)^4 \exp \left(c_4 \xi(\varepsilon)^2 \sqrt{\ell\left(\left\lceil \varepsilon_1^3 T\right\rceil\right)-1}\right)} .
$$

Now, the lazy condition \ref{eq:lazy_cond_2} ensures that $\lim _{t \rightarrow \infty} \ell(t) / t^\gamma>0$ for some $\gamma \in(0,1)$ and $\ell(t) \leq t$. Thus there exists $T_4 \geq T_3$ such that for all $T \geq T_4$,
$$
\mathbb{P}\left(\mathcal{E}_{2,\left\lceil\varepsilon_1^3 T\right\rceil}^c\right) \lesssim \frac{\ell\left(\left(\left\lceil\varepsilon_1^3 T\right\rceil\right)-1\right)^{d / 2+1}}{\xi(\varepsilon)^4 \exp \left(c_4 \xi(\varepsilon)^2 \sqrt{\ell\left(\left\lceil\varepsilon_1^3 T\right\rceil\right)-1}\right)} \lesssim \frac{T^{d / 2+1}}{\exp \left(c_5(\varepsilon) T^{\gamma / 2}\right)} .
$$

This shows that
$$
\begin{aligned}
\sum_{T=T_0 \vee T_1}^{\infty} \mathbb{P}\left(\mathcal{E}_{2,\left\lceil\varepsilon_1^3 T\right\rceil}^c\right) & =\sum_{T=T_0 \vee T_1}^{T_4} \mathbb{P}\left(\mathcal{E}_{2,\left\lceil\varepsilon_1^3 T\right\rceil}^c\right)+\sum_{T=T_4+1}^{\infty} \mathbb{P}\left(\mathcal{E}_{2,\left\lceil\varepsilon_1^3 T\right\rceil}^c\right) \\
& \lesssim \sum_{T=T_0 \vee T_1}^{T_4} \mathbb{P}\left(\mathcal{E}_{2,\left\lceil\varepsilon_1^3 T\right\rceil}^c\right)+\sum_{T=T_4+1}^{\infty} \frac{T^{d / 2+1}}{\exp \left(c_5(\varepsilon) T^{\gamma / 2}\right)} \\
& <\infty
\end{aligned}
$$
where the last inequality follows from the fact that we can upper bound the infinite sum by a Gamma function, which is convergent as long as $\gamma>0$.

Finally, we have thus shown that
$$
\sum_{T=T_0 \vee T_1+1}^{\infty} \mathbb{P}\left(\mathcal{E}_T^c\right)<\infty
$$

We note that this infinite sum depends on $(\ell(t))_{t \geq 1}$ and $\varepsilon$ only.

\textbf{Last step:} Finally, we have shown that for all $\varepsilon>0$
$$
\mathbb{E}[\tau] \leq \frac{c_1}{1-\varepsilon} T_\theta^{\star} \log (1 / \delta)+o(\log (1 / \delta))+T_0 \vee T_1+\sum_{T=T_0 \vee T_1}^{\infty} \mathbb{P}\left(\mathcal{E}_T^c\right)
$$
where $\sum_{T=T_0 \vee T_1}^{\infty} \mathbb{P}\left(\mathcal{E}_T^c\right)<\infty$ and is independent of $\delta$. Hence,
$$
\limsup _{\delta \rightarrow 0} \frac{\mathbb{E}\left[\tau_\delta\right]}{\log (1 / \delta)} \leq \frac{c_1}{1-\varepsilon} T_\theta^{\star}
$$

Letting $\varepsilon$ tend to 0 and recalling that $c_1 \lesssim \sigma^2$, we conclude that
$$
\limsup _{\delta \rightarrow 0} \frac{\mathbb{E}\left[\tau_\delta\right]}{\log (1 / \delta)} \lesssim \sigma^2 T_\theta^{\star}
$$

\end{proof}

\end{document}

%% file: macro.tex
\usepackage{dsfont}

\def\bbE{\mathbb{E}}

\def\bbP{\mathbb{P}}

\def\bbR{\mathbb{R}}

\def\cA{\mathcal{A}}

\def\cE{\mathcal{E}}

\def\cO{\mathcal{O}}

\def\cX{\mathcal{X}}

\def\det{\mathrm{det}}

\def\exp{\mathrm{exp}}